\documentclass{article}
\usepackage{spconf}
\usepackage{mathtools}
\usepackage{multirow}
\usepackage{hhline}
\usepackage{arydshln}
\usepackage{color}
\usepackage{hyperref}

\usepackage{amsmath,amssymb,amsfonts,amsthm,bbm, stmaryrd,  float}
\usepackage{epic,eepic,epsfig,longtable}
\usepackage{multirow,verbatim}
\usepackage{array}
\usepackage{verbatim}
\usepackage{epsfig}
\usepackage{setspace}
\usepackage{CJK}
\usepackage{color}
\usepackage[linesnumbered,ruled,vlined,commentsnumbered]{algorithm2e}

\usepackage{caption}
\usepackage[hang,flushmargin]{footmisc}
\usepackage{lipsum}
\newcommand{\algorithmfootnote}[2][\footnotesize]{%
  \let\old@algocf@finish\@algocf@finish% Store algorithm finish macro
  \def\@algocf@finish{\old@algocf@finish% Update finish macro to insert "footnote"
    \leavevmode\rlap{\begin{minipage}{\linewidth}
    #1#2
    \end{minipage}}%
  }%
}

\numberwithin{equation}{section}

\renewcommand{\hat}{\widehat}

\newcommand{\bfm}[1]{\ensuremath{\mathbf{#1}}}

          \def\cO{{\cal  O}}     
     \def\bP{\bfm P}

     \def\bU{\bfm U}          
\def\bv{\bfm v}     \def\bV{\bfm V}          
     \def\bW{\bfm W}          
     \def\bX{\bfm X}

\def\bzero{\bfm 0}

\newcommand{\bfsym}[1]{\ensuremath{\boldsymbol{#1}}}

       \def \bLambda  {\bfsym{\Lambda}}
          
\def \bSigma   {\bfsym{\Sigma}}       
         
\def \bOmega   {\bfsym{\Omega}}

\renewcommand{\hat}{\widehat}

\def \heps     {\hat{\heps}}

\DeclareMathOperator*{\argmin}{argmin}

\DeclareMathOperator{\diag}{diag}
\DeclareMathOperator{\E}{E}

 at 8truept

\def\today{\ifcase\month\or
  January\or February\or March\or April\or May\or June\or
  July\or August\or September\or October\or November\or December\fi
  \space\number\day, \number\year}

\def \newpage {\vfill\eject}

\newdimen\biblioindent\biblioindent=30pt
\newcommand{\beq}  {\begin{equation}}
\newcommand{\eeq}  {\end{equation}}
\newcommand{\beqn} {\begin{eqnarray}}
\newcommand{\eeqn} {\end{eqnarray}}
\newcommand{\beqnn}{\begin{eqnarray*}}
\newcommand{\eeqnn}{\end{eqnarray*}}

\allowdisplaybreaks
\newtheorem{lem}{Lemma}

\newtheorem{thm}{Theorem}
\newcounter{CondCounter}
\newcommand{\fnorm}[1]{\lVert#1\rVert_{F}}

\def \col {\text{Col}}

\def \RR	{\mathbb{R}}

\def\R{{\mathbb R}}

\def\E{{\mathbb E}}

\def\diag{{\rm diag}}
\def\Tr{{\rm Tr}}

\def\O{{\cal{O}}}

\def\argmin{{\rm argmin}}

\numberwithin{equation}{section}
\theoremstyle{plain}

\theoremstyle{definition}

\title{Online Distributed Estimation of Principal Eigenspaces}
\name{Davoud Ataee Tarzanagh, Mohamad Kazem Shirani Faradonbeh, George Michailidis}
\address{Informatics Institute, University of Florida, Gainesville, FL, USA, 32611}
\begin{document}
%\ninept
%
\maketitle
\begin{abstract}
	Principal components analysis (PCA) is a widely used dimension reduction technique with an extensive range of applications. In this paper, an {\em online distributed} algorithm is proposed for recovering the principal eigenspaces. We further establish its rate of convergence and show how it relates to the number of nodes employed in the distributed computation, the effective rank of the data matrix under consideration, and the gap in the spectrum of the underlying population covariance matrix. The proposed algorithm is illustrated on low-rank approximation and $\boldsymbol{k}$-means clustering tasks. The numerical results show a substantial computational speed-up vis-a-vis standard distributed PCA algorithms, without compromising learning accuracy.
\end{abstract}
\section{Introduction}\label{sec:intro}

Principal components analysis constitutes a fundamental technique for reducing the dimension of the data in numerous areas. It extracts a $K$ dimensional (affine) subspace that accounts for most of the variation in the data, thus revealing its low-dimensional structure. Further, it can serve as a preprocessing step to reduce the data dimension in various machine learning tasks including regression analysis, $\boldsymbol{k}$-means clustering \cite{kanungo2002efficient}, Non-negative Matrix Factorization (NMF)~\cite{lee1999learning}, and Latent Dirichlet Allocation (LDA) models~\cite{blei2003latent}. 

In many applications, large data sets are distributed across different locations. This leads to some new challenges in both the design and analysis of the learning procedures. Indeed, different considerations need to be carefully addressed including local data collection, privacy/sharing issues, reliability, and communication costs. To address these issues, one needs to design and analyze distributed estimation procedures. A typical method computes low-rank approximations for subsets of the full data, and aggregates the resulting estimates at a central node~\cite{wu2018review}. A number of such techniques have been presented in the literature, including distributed sketching~\cite{qu2002principal,liang2014improved,kannan2014principal,chen2016integrating,tarzanagh2018fast} and estimation~\cite{el2010second,fan2017distributed} methods. Recently, Fan et al.~\cite{fan2017distributed} studied a distributed estimation algorithm, wherein each local node computes the top $K$ eigenvectors for its own subset of the data, and then transmits them to a central node. Then, the central node runs an eigen-decomposition algorithm on the matrix comprising of all the eigenvectors received from the local nodes. Further, the estimation accuracy of the above method compared to PCA of the entire data is provided in the aforementioned reference.

In contrast to previous work on distributed estimation of principal eigenspaces, we are interested in an {\em online} scheme where the data becomes available in a sequential manner. We propose an {\em online distributed} algorithm consisting of two aggregation steps. In the first step, the incoming batch of the data is further split into subsets and allocated to different computation nodes. Each node computes a low-rank approximation of the corresponding subset, followed by a \emph{local} aggregation step to obtain an estimate of the principal eigenspaces of the current batch. Then, the final estimate is obtained by passing the intermediate batch results to a \textit{fusion} center that aggregates across all batches (global aggregation). Hence, the key contribution of this work is to introduce a fast distributed PCA algorithm, capable of providing high quality results for \emph{streaming} data.

The remainder of the paper is organized as follows. Section \ref{sec:2} introduces the problem of estimating the principal eigenspaces. Then, we discuss algorithms for distributed, as well as \emph{online and distributed} estimation of principal eigenspaces in Section \ref{sec:3}. Section \ref{sec:4} establishes theoretical properties, while Section \ref{sec:5} illustrates the methodology on selected real data sets. 

\section{Problem Formulation} \label{sec:2}

The objective is to learn the $K$ dimensional principal eigenspace of the \emph{unknown} population covariance matrix $\bSigma$ of the data. Technically, we want to estimate the subspace spanned by the eigenvectors corresponding to the top $K$ eigenvalues of $\bSigma \in \RR^{d \times d}$, for a given $K \in [d]\equiv\left\{ 1, \cdots, d\right\}$. The data consists of \emph{streaming} i.i.d samples $\{\bX_i \left(t \right) \}_{i=1}^N\subseteq \RR^d$, being observed over the time interval $1 \leq t \leq T$. Hence, the positive semidefinite matrix $\bSigma$ corresponds to the population covariance matrix of the centered samples; $\E \bX_1\left(t\right)=\bzero$, $\E \left[\bX_1\left(t\right)\bX_1^\top\left(t\right)\right]= \bSigma$. 

We further assume that the samples are generated from a sub-Gaussian distribution; $\| \| \bX_1\left(t\right) \|_2 \|_{\psi_2} <\infty$, where $\| \cdot \|_{\psi_i}$ denotes the Orlicz norm: $\|X\|_{\psi_i}= \sup_{p\ge 1}(\E|X|^p)^{1/p} p^{-1/i}$. To proceed, consider the eigenvalue decomposition $\bSigma= \bV\bLambda\bV^\top$, where $\bLambda= \diag(\lambda_1, \lambda_2, \cdots, \lambda_d)$ with $\lambda_1 \ge \lambda_2 \ge \cdots \ge \lambda_d$, and $\bV= (\bv_1, \cdots, \bv_d)\in \O_{d\times d}$ (the set of $d \times d$ orthonormal matrices). Further, let $\bV_K=(\bv_1, \cdots, \bv_K)$ denote the orthonormal matrix comprising of the $K$ eigenvectors corresponding to the largest $K$ eigenvalues $\lambda_j, j=1,\cdots,K$. Hence, we want to identify the subspace $\col(\bV_K)$, utilizing the streaming data $\{\bX_i \left(t\right) \}_{i=1}^N$.

In order to ensure that $\col(\bV_K)$ is identifiable, we assume that there is a gap between the eigenvalues; $\delta := \lambda_K - \lambda_{K+1}>0$. Throughout the paper, let $\kappa := \lambda_1/\delta$ and $r=r(\bSigma):= \Tr(\bSigma)/\lambda_1$ denote the condition number and the effective rank of $\bSigma$, respectively. 

\section{Algorithm} \label{sec:3}

In the offline (i.e. non-streaming) setting where $T=1$, the standard procedure for estimating $\col(\bV_K)$ is to use the $K$ leading eigenvectors of the empirical covariance matrix\footnote{For $T=1$, we drop the trivial time indicator $(1)$.} $\hat{\bSigma}=N^{-1}\sum_{i=1}^{N}\bX_i \bX_i^\top$. Let $\hat\bSigma= \hat\bV \hat\bLambda {\hat\bV}^\top$ be the eigenvalue decomposition of $\hat\bSigma$, analogous to that of $\bSigma$ defined in the previous section. Then, $\col(\hat\bV_K)$ is being used to estimate $\col(\bV_K)$. 

Next, we briefly discuss the Distributed PCA algorithm (\textsc{Dpca})~\cite{fan2017distributed}. In order to estimate the principal eigenspace $\col(\bV_K)$ in a distributed manner, the data is split across $m$ nodes, each possessing ${n}$ samples; i.e. $N=m{n}$. For each $\ell \in [m]$, the projection matrix $\hat{\bV}_K^{(\ell)}$ corresponding to the top $K$ empirical eigenvectors of node $\ell$ are formed. Then, $\bar{\bSigma}= m^{-1}\sum\limits_{i=1}^m \hat{\bV}_K^{(\ell)} \hat{\bV}_K^{(\ell)^\top}$ is calculated, followed by its eigenvalue decomposition to obtain the $K$ leading eigenvectors of $\bar{\bSigma}$, denoted by $\bar{\bV}_{K}=(\bar{\bv}^{(\ell)}_1, \cdots, \bar{\bv}^{(\ell)}_K)$. 

Clearly, the communication cost of \textsc{Dpca} is $O(NKd/{n})$. Indeed, every node needs to send the $d \times K$ matrix $\hat{\bV}^{(\ell)}_K$ to the central aggregation node. Thus, for large data sets, in addition to possible storage issues, there is a risk of communication congestion. These issues are further compounded in the presence of streaming data.
\begin{algorithm}[t]
	\caption{Online Distributed PCA (\textsc{Odpca}). }
	\label{algo:2}
	\SetKwInOut{Input}{Input}
	\SetKwInOut{Output}{Output}
	\BlankLine
	Initialize $\widetilde{\bSigma}(0) =0_{d \times d}$ \;
	\For{$t\leftarrow 1$ \KwTo $T$}{\label{forins}
		\Input{online data $\{\bX_i^{(\ell)}(t)\}_{i \in [n], \ell \in [m]}$;}
		\For{$\ell\leftarrow 1$ \KwTo $m$}{
			Compute $\hat{\bV}^{(\ell)}_K(t)$; the $K$ leading eigenvectors of $\hat{\bSigma}^{(\ell)}(t)=n^{-1}\sum_{i=1}^{n}\bX_i^{(\ell)}(t)\bX^{(\ell)\top}_{i}(t)$\;
			%Send $\hat{\bV}^{(\ell)}_K(t)$ to the local processor\;
		}
		Compute $\bar{\bV}_K(t)$; the $K$ leading eigenvectors of $\bar{\bSigma}(t)=m^{-1}\sum_{\ell=1}^{m}\hat{\bV}^{(\ell)}_K(t)\hat{\bV}^{(\ell)\top}_K(t)$\;
		Update $\widetilde{\bSigma}(t)= \widetilde{\bSigma}(t-1)+T^{-1}\bar{\bV}_K(t)\bar{\bV}^\top_K(t)$\;
	}
	\BlankLine
	Let $\{\widetilde{\bv}_{j}(T)\}_{j=1}^K$ be the top $K$ eigenvectors of $\widetilde{\bSigma}(T)$\;
	\Output{~$\widetilde{\bV}_{K}(T)=(\widetilde{\bv}_{1}(T),\cdots,\widetilde{\bv}_{K}(T))\in\R^{d\times K}$.}
\end{algorithm}

To address the above issues, we propose Algorithm~\ref{algo:2} (\textsc{Odpca}) for estimating the principal eigenspaces in an {\em online and distributed} fashion. Let $N=T m n$ denote the total number of samples, and let $\{\bX_i^{(\ell)}(t)\}_{i \in [n]}$ be the batch of data that becomes available to node $\ell \in [m]$, at time $t \in [T]$. In the first round of aggregation all local nodes contribute to obtain the approximation $\bar{\bSigma}(t)$ at time $t$. Then, the final estimate across all data batches, $\widetilde{\bV}_K(T)$, is calculated by the fusion center. The communication cost for the local processors is divided by the number of time steps; so it is $O( NKd/ \left(nT\right) )$. 

In order to analyze \textsc{Odpca}, first define $\Delta(\bU,\bV) :=\|\bU \bU^\top-\bV \bV^\top\|_{F}$, where $\| \cdot \|_{F}$ denotes the Frobenius norm, and  $\bU,\bV$ are arbitrary projection matrices with the same number of rows. The metric $\Delta$ is well-defined, and is commonly used in the literature to reflect the distance between subspaces, as well as the projections on them~\cite{yu2014useful,fan2017distributed}. Next, consider the online problem \cite{nazari2019dadam} $\min\limits_{\bU\in\cO_{d\times K}} H(\bU),$ where $H(\bU) := (T m)^{-1} \sum_{t=1}^{T}\sum_{\ell=1}^{m}\Delta^2(\bU,\widehat{\bV}_K^{(\ell)}(t))$. Intuitively, a minimizer $\bU^*$ is a ``center" of $\{ \hat{\bV}^{(\ell)}_K(t) \}_{ \ell \in [m], t \in[T]}$, which is by the following result a projection matrix on the principal eigenspace of $\bOmega := (T m)^{-1} \sum\limits_{t=1}^{T}\sum\limits_{\ell=1}^{m} \widehat{\bV}_K^{(\ell)}(t) {\widehat{\bV}_K^{(\ell)}(t)}^\top$. 
\begin{lem}\label{lem-loss}
The matrix $\bU^* \in \cO_{d\times K} $ consists of the top $K$ eigenvectors of $\bOmega$, if and only if  ${\bU}^*  \in \argmin_{\bU\in\cO_{d\times K}} H(\bU)$.
\end{lem}
\begin{proof}
Letting $\widehat{\bP}^{(\ell)}(t)=\widehat{\bV}_K^{(\ell)}(t) \widehat{\bV}_K^{(\ell)\top}(t)$, we have $H(\bU) = \| \bOmega \|_{F}^2 + \|\bU\bU^\top\|_{F}^2 
	+ (T m)^{-1} \sum_{t=1}^{T}\sum_{\ell=1}^{m} \|\bOmega - \hat{\bP}^{(\ell)}(t)\|_{F}^2 - 2 \Tr(\bU \bU^\top \bOmega )$. \\Then, $\|\bU\bU^\top\|_{F}^2$ is constant since $\bU$ is orthonormal. Thus, a minimizer of $H(\cdot)$ maximizes $\Tr(\bU \bU^\top\bOmega )= \Tr( \bU^\top \bOmega \bU)$; i.e. is the projection matrix to a principal subspace.
\end{proof}
In the next section, we show that Algorithm~\ref{algo:2} provides an accurate approximation of $\bU^*$. Indeed, the intermediate steps of calculating $\bar{\bV}_K(t)$ lead to significantly faster computations with negligible growth in the statistical error.
 
\section{Theoretical Analysis} \label{sec:4}
Next, we analyze the learning error in terms of the distance $\Delta(\widetilde{\bV}_{K}(T), \bV_K)$ defined in the previous section. The main result states that the statistical error rates are the same as using the full sample $\{\bX_i^{(\ell)}(t)\}_{i \in [n], \ell \in [m],t \in [T]}$, as long as the sample size $n$ of each node $\ell$ at every time $t$ is large enough. Henceforth, we focus on the centralized error $\Delta(\widetilde{\bV}_{K}(T), \widetilde{\bW}_K)$, where $\widetilde{\bV}_{K}(T)$ is the output of Algorithm~\ref{algo:2} (i.e. principal eigenvectors of $\widetilde{\bSigma}(T)$), and $\widetilde{\bW}_K$ consists of the top $K$ eigenvectors of $\E \left[ \widetilde{\bSigma}(T) \right]$. A comprehensive analysis of the deterministic bias $\Delta(\widetilde{\bW}_K, \bV_K)$ for both symmetric and asymmetric distributions is established before~\cite{fan2017distributed}.

The following result addresses the behavior of the statistical error. Note that if $n$ is large enough, the assumptions of Theorem~\ref{thm:1} will be satisfied~\cite{fan2017distributed}.
\begin{thm}\label{thm:1}
Suppose that in Algorithm~\ref{algo:2}, $n\geq r$, and $\| \bOmega^* -\bV_K \bV_K^\top \|_2 < 1/4$, where $\bOmega^*:=\E \left[\hat{\bV}_K^{(\ell)}(T) {\hat{\bV}_K^{(\ell)}(T)}^\top \right]$. Them, we have\footnote{the notation ``$\lesssim$" states that ``$\leq$" holds, modulo a universal constant.}
\begin{eqnarray*}
	\left\| \Delta(\widetilde{\bV}_{K}(T), \widetilde{\bW}_K) \right\|_{\psi_1} \lesssim \kappa \sqrt{\frac{Kr}{Tmn}}.
\end{eqnarray*}
\end{thm}
\begin{proof}
Let $\bV^*_K$ contain the top $K$ eigenvectors of $\bOmega^*$, and denote $\alpha= \kappa \sqrt{Kr (mn)^{-1}}$. Using  Theorem~1 in~\cite{fan2017distributed}, $\| \bOmega^* -\bV_K \bV_K^\top \|_2 < 1/4$ implies that $\left\| \Delta(\bar{\bV}_{K}(t), \bV_K^*) \right\|_{\psi_1} \lesssim \alpha$, for all $t \in [T]$. Further, applying Jensen's inequality to the definition of $\left\| \cdot \right\|_{\psi_1}$, we get $ \left\| \E \left[ \bar{\bV}_{K}(t) \bar{\bV}_{K}(t)^\top- \bV_K^* {\bV_K^*}^\top \right] \right\|_{F} \lesssim \alpha$. Therefore, by the triangle inequality, $\lambda_K \left(\E \widetilde{\bSigma}(T)\right) -\lambda_{K+1}\left(\E \widetilde{\bSigma}(T)\right)$ is a positive number bounded away from zero, as long as $mn$ is sufficiently large. Then, by the triangle inequality, $ \left\| \fnorm{\bar{\bV}_{K}(t) \bar{\bV}_{K}(t)^\top - \E \left[\bar{\bV}_{K}(t) \bar{\bV}_{K}(t)^\top\right]} \right\|_{\psi_1} \lesssim \alpha$ holds, which according to Lemma~4 in~\cite{fan2017distributed} implies that 
\begin{eqnarray} \label{proofeq1}
\left\| \fnorm{\widetilde{\bSigma}(T) - \E \widetilde{\bSigma}(T)} \right\|_{\psi_1} \lesssim \frac{\alpha}{\sqrt{T}}. 
\end{eqnarray}
Thus, since the eigengap of $\E \widetilde{\bSigma}(T)$ is bounded from below, according to Davis-Kahan’s Theorem (Theorem~2 in~\cite{yu2014useful}) \eqref{proofeq1} implies the desired result.
\end{proof}

\section{Real-Data Experiments} \label{sec:5}
Next, we compare the performance of \textsc{Odpca} and \textsc{Dpca}~\cite{fan2017distributed}. The experiments correspond to solving rank--$K$ approximation and $\boldsymbol{k}$-means clustering tasks. We set $K=10$ for rank--$K$ approximation, and $k=10$ for $\boldsymbol{k}$-means clustering. The data sets being used in this section are as follows\footnote{Obtained from the UCI Machine Learning Repository: \url{http://archive.ics.uci.edu/ml}.}: \textbf{(i)} NewsGroups ($N=18774,d={61188}$); \textbf{(ii)} MNIST ($N=7 \times 10^4, d=784$); \textbf{(iii)} BOW NYTimes ($N=3\times 10^5 , d= 102660$). 
\begin{figure}[t]
	\centering \makebox[0in]{
		\begin{tabular}{cc}
			\includegraphics[scale=0.092]{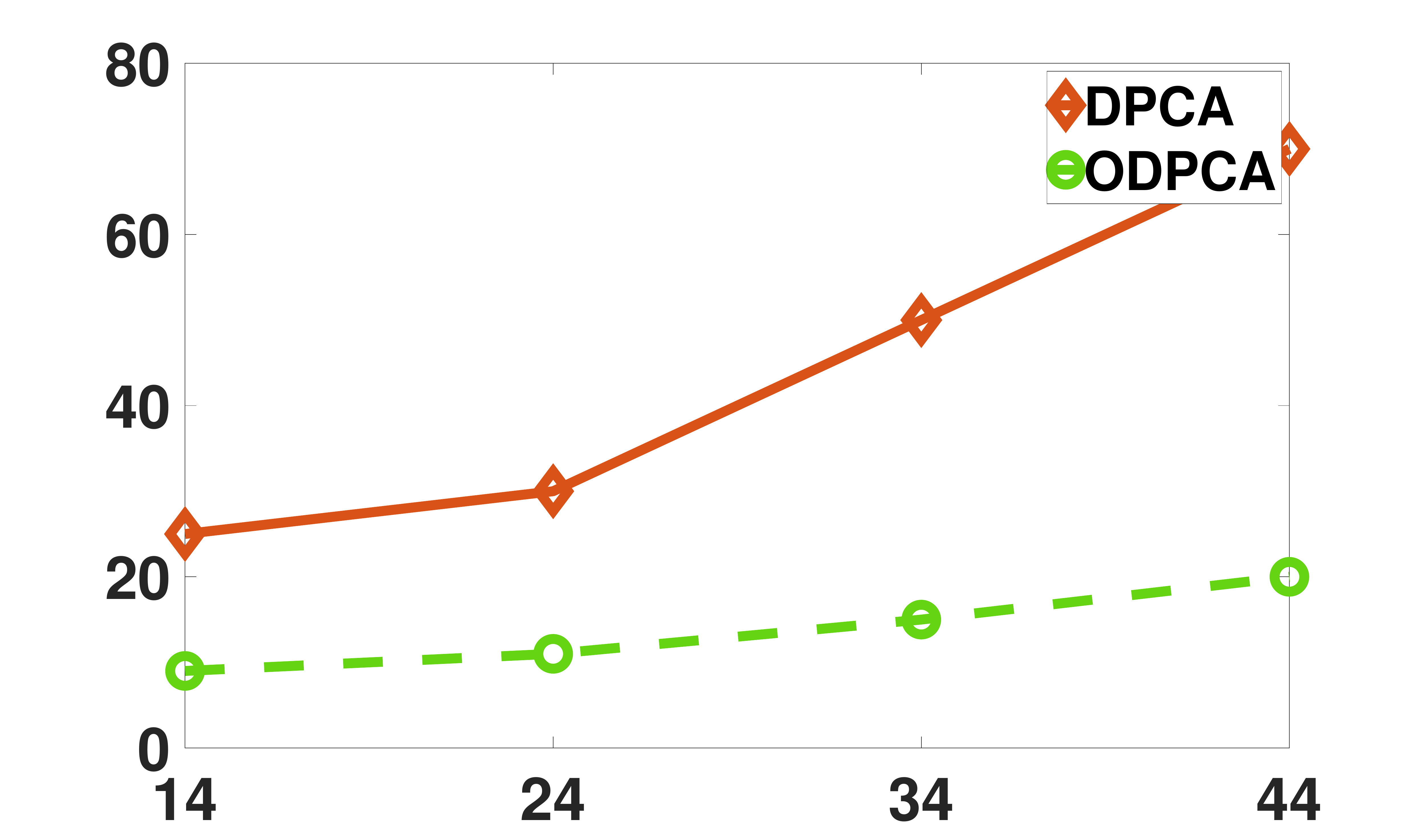}
			\includegraphics[scale=0.092]{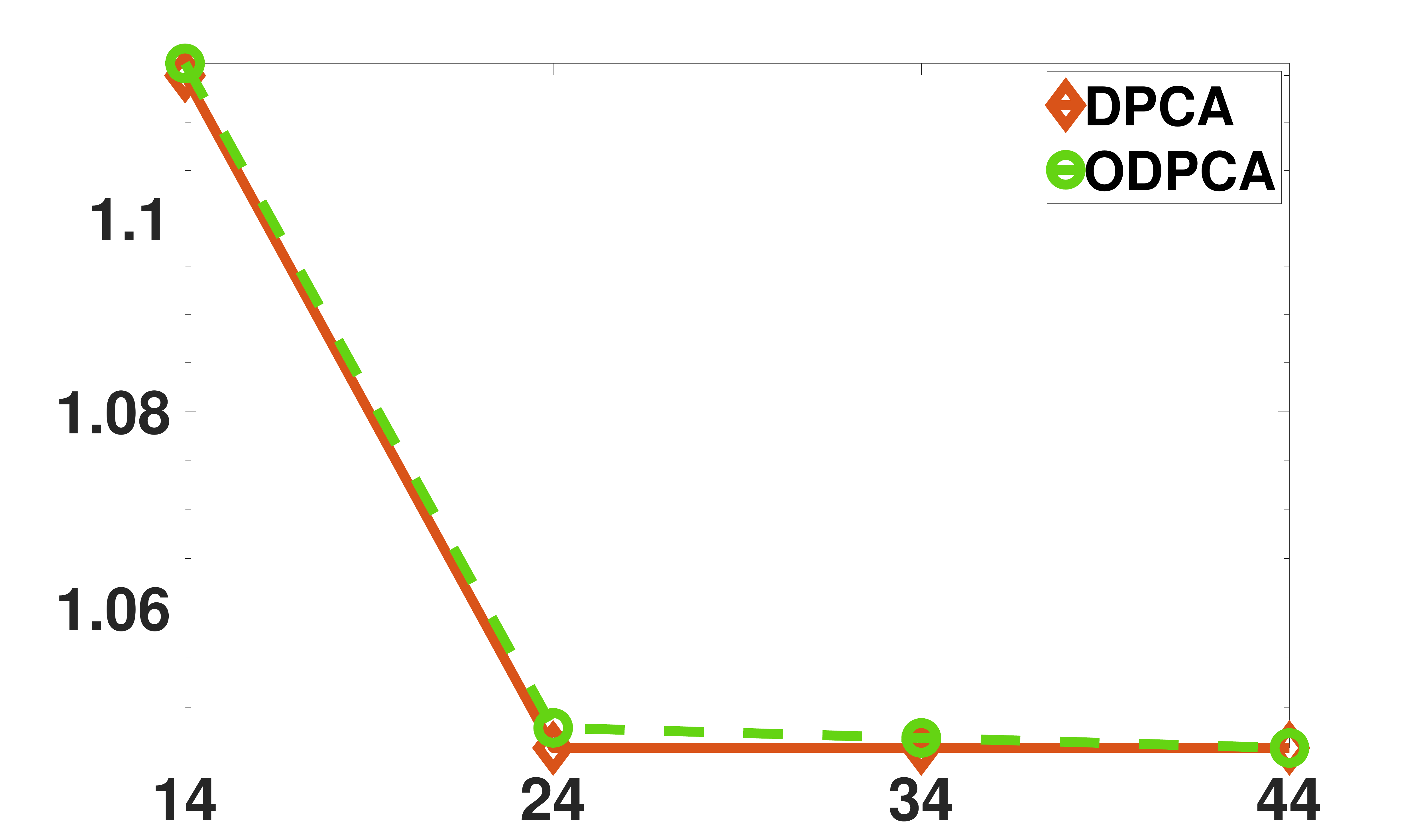}
			\\
			(a) ~ MNIST
			\\
			\includegraphics[scale=0.092]{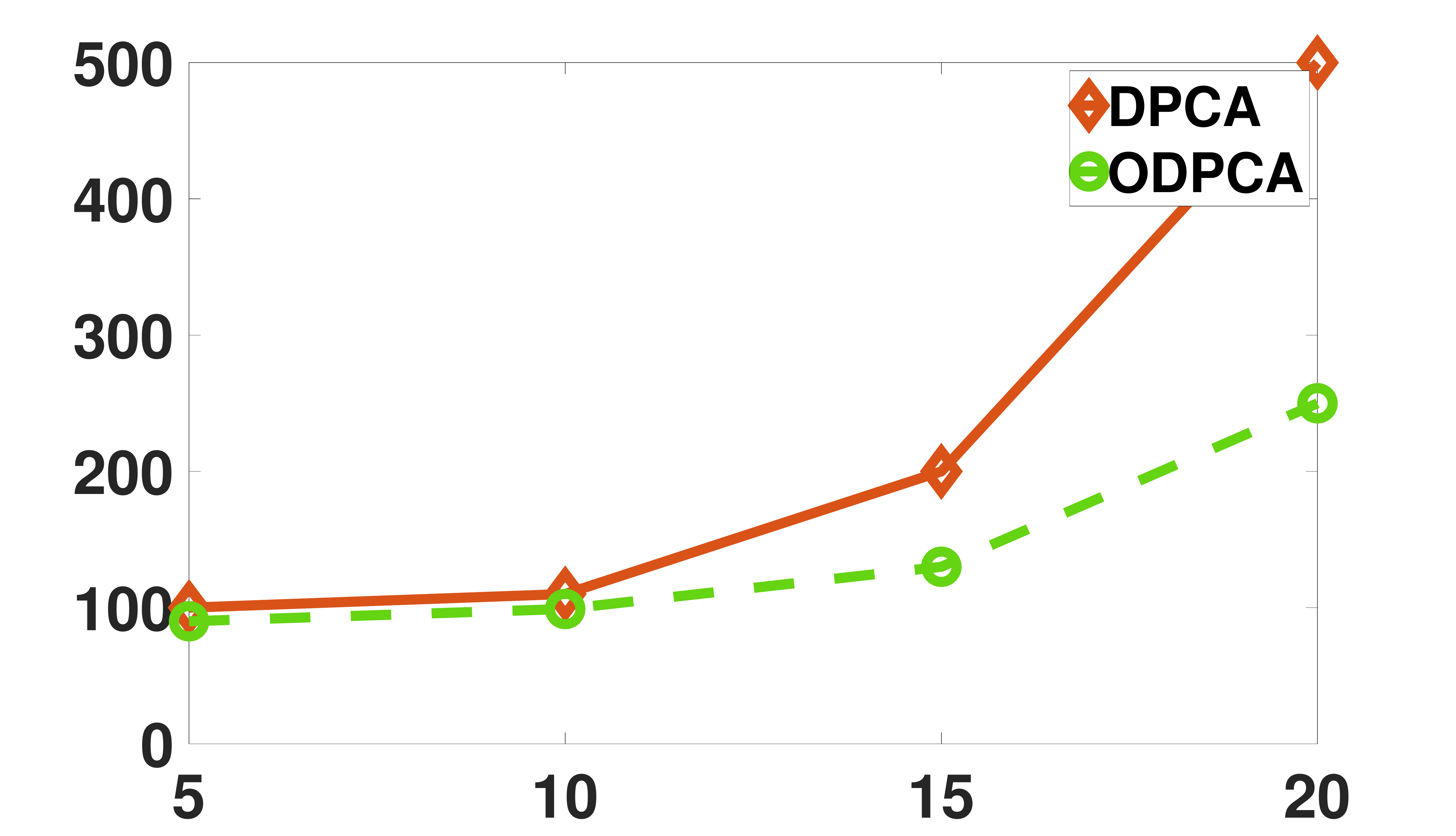}
			\includegraphics[scale=0.092]{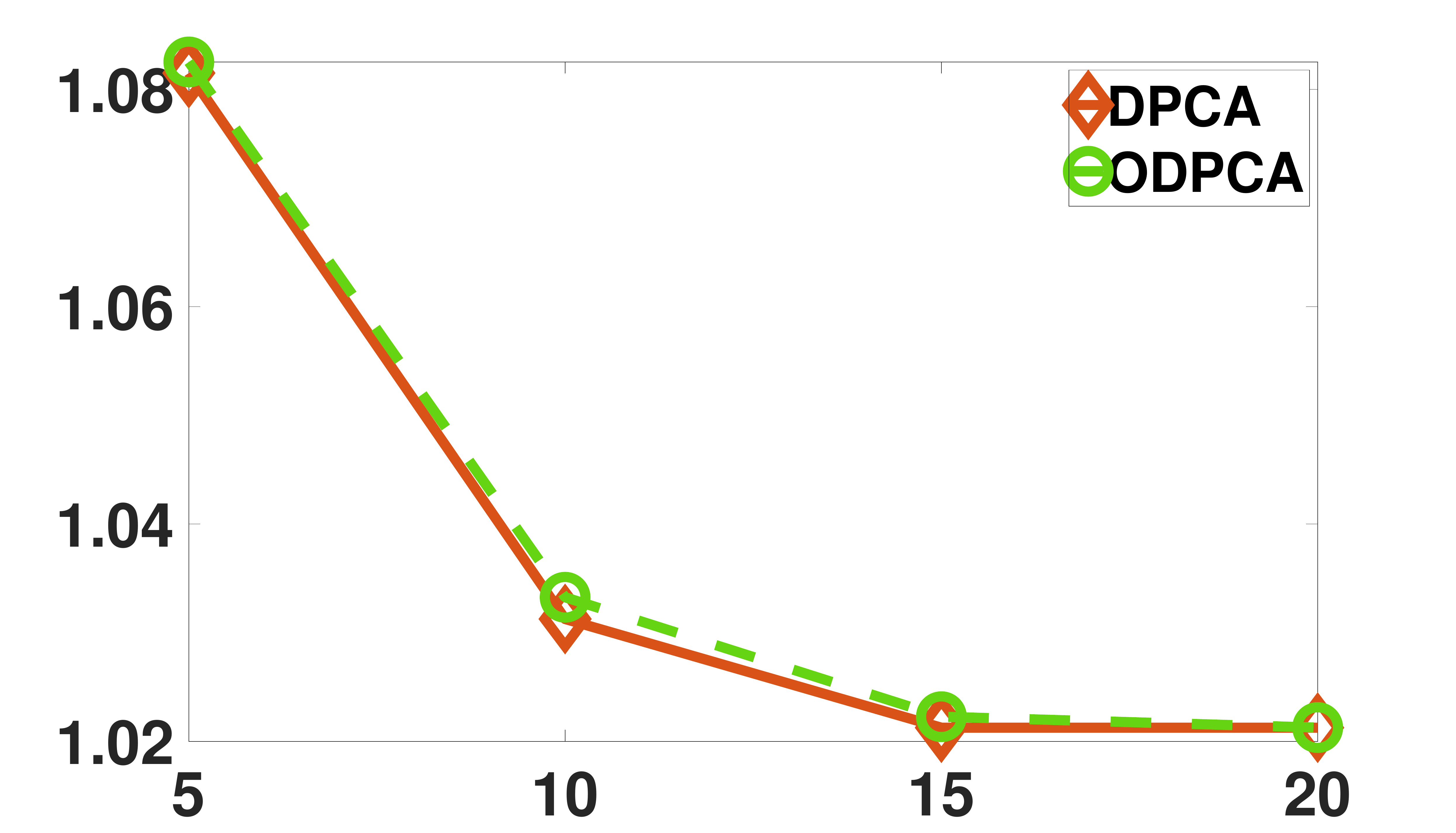}
			\\
			(b) ~ NewsGroups 
			\\
		\end{tabular}
	}
	\caption{Low-rank approximation: running time (left), and relative estimation error (right), v.s. projection dimension.}\label{fig:lowrank}
\end{figure}
\begin{figure}[t]
	\centering \makebox[0in]{
\begin{tabular}{cc}
\includegraphics[scale=0.089]{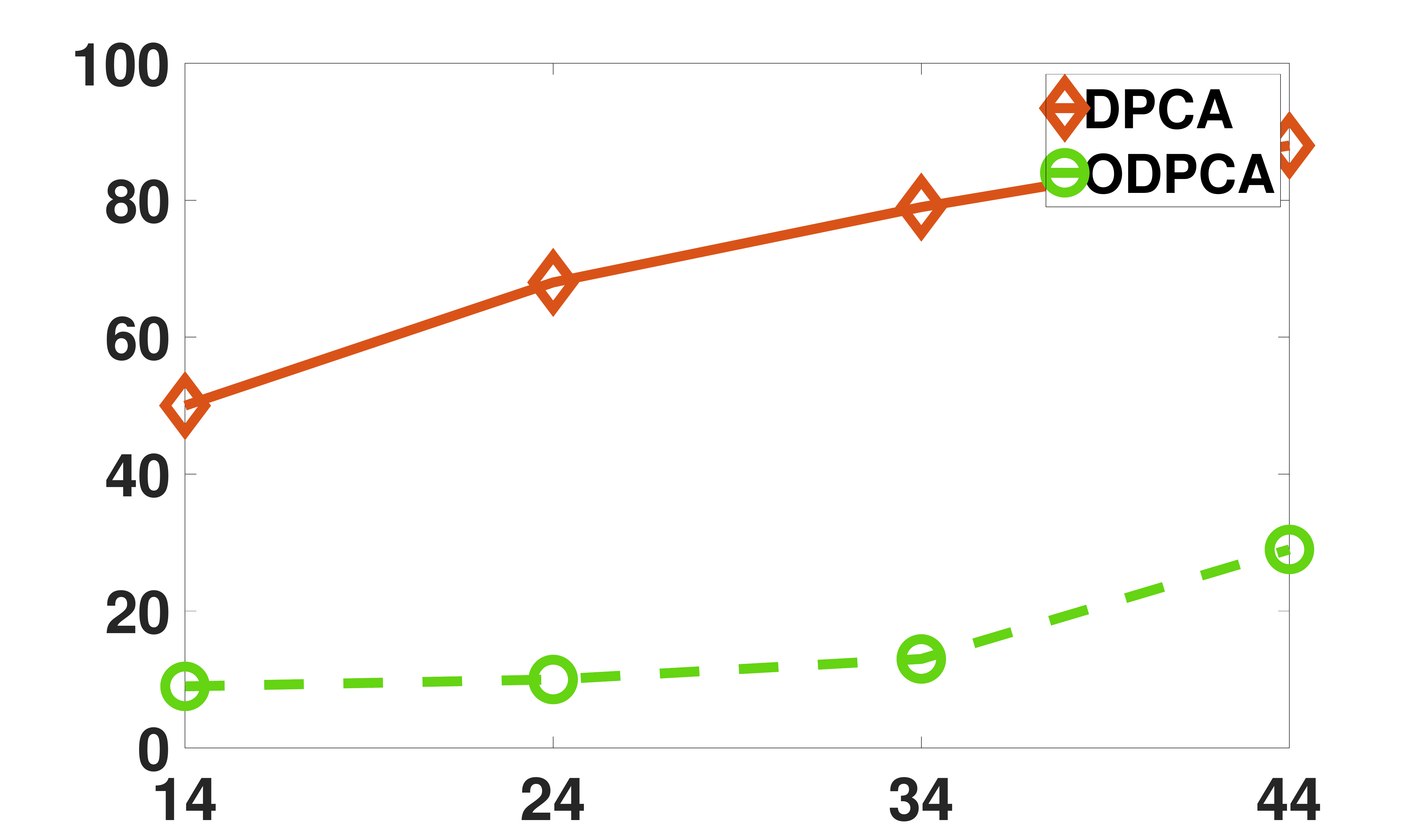}
\includegraphics[scale=0.089]{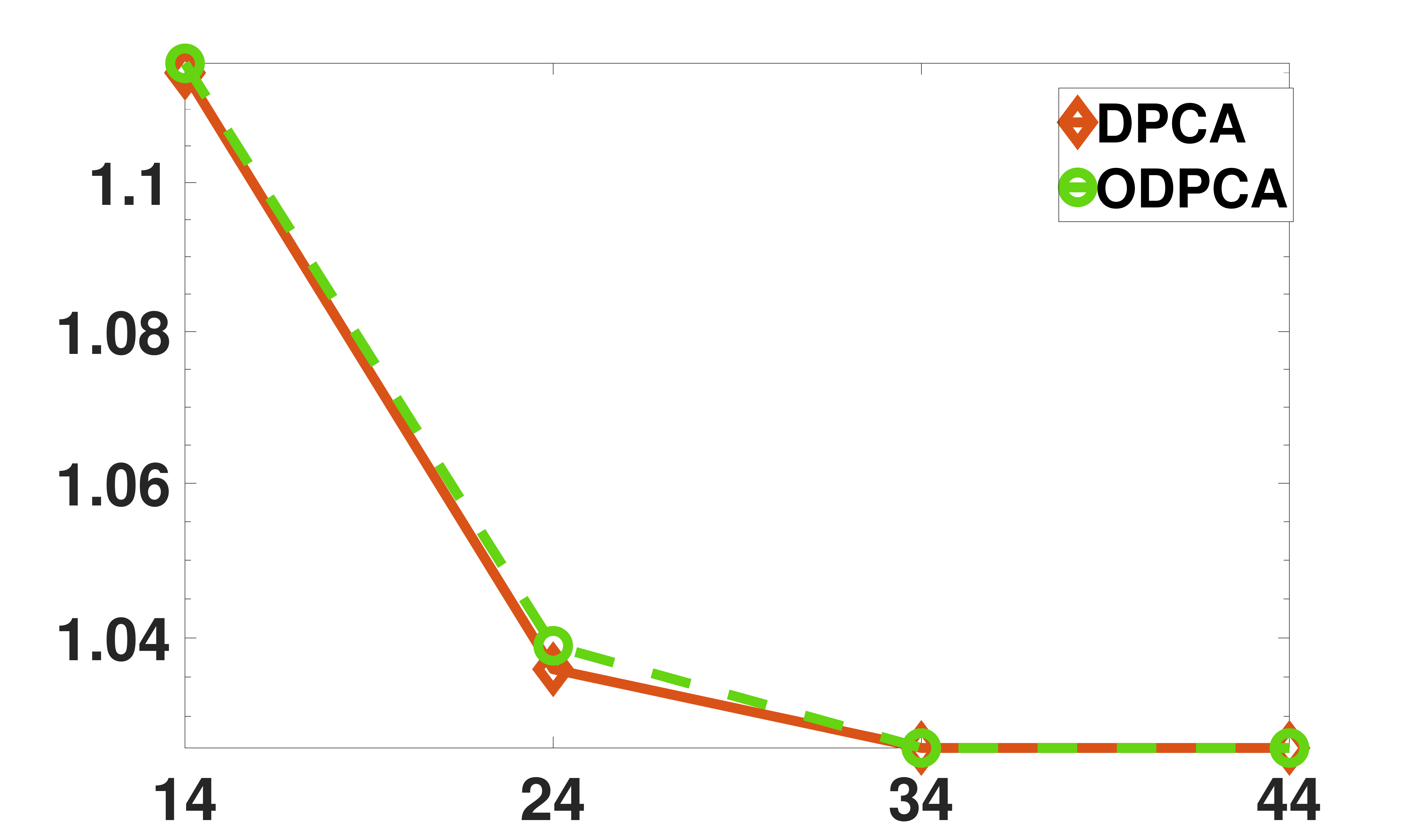}
\\
(a) ~ MNIST
\\
\includegraphics[scale=0.089]{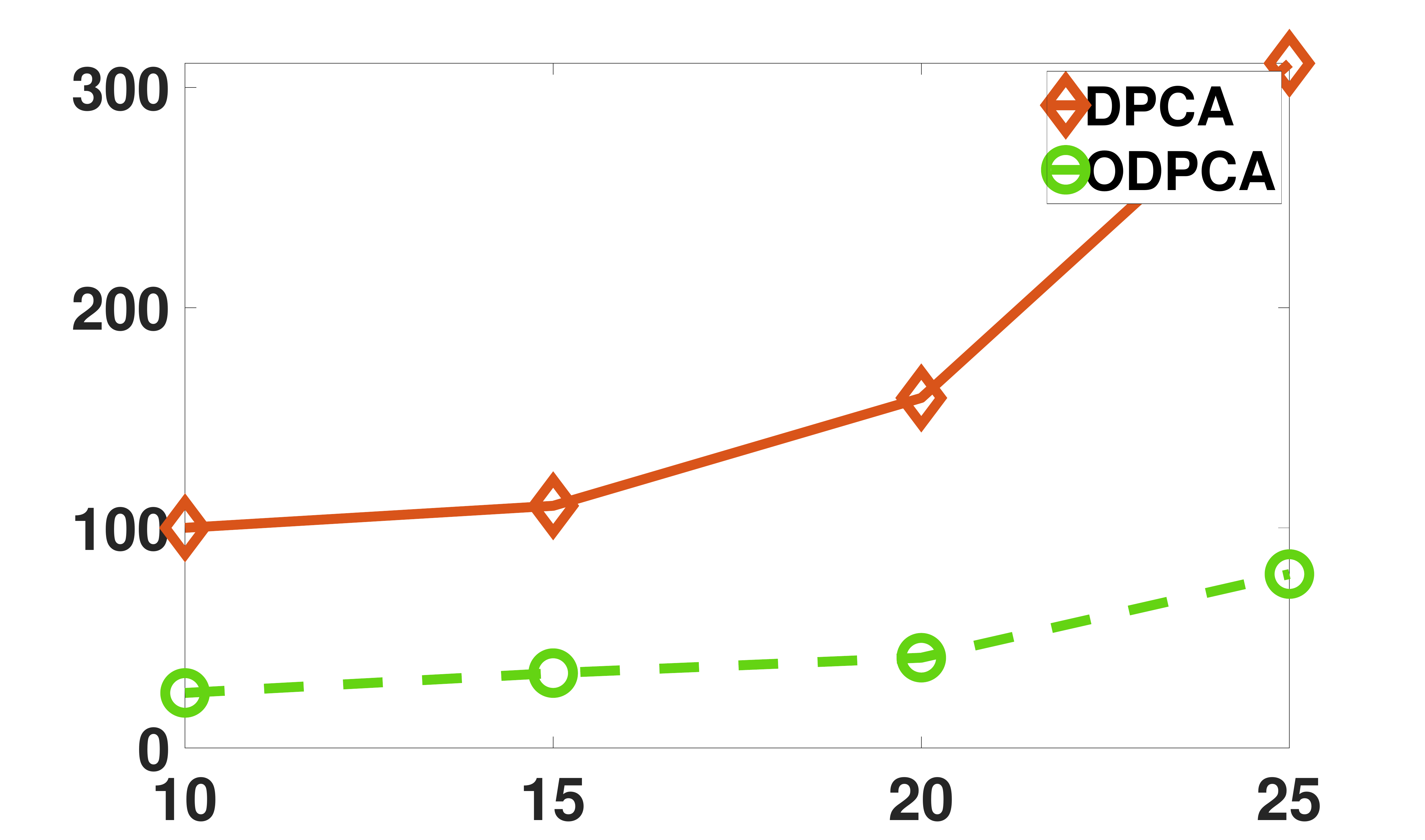}
\includegraphics[scale=0.089]{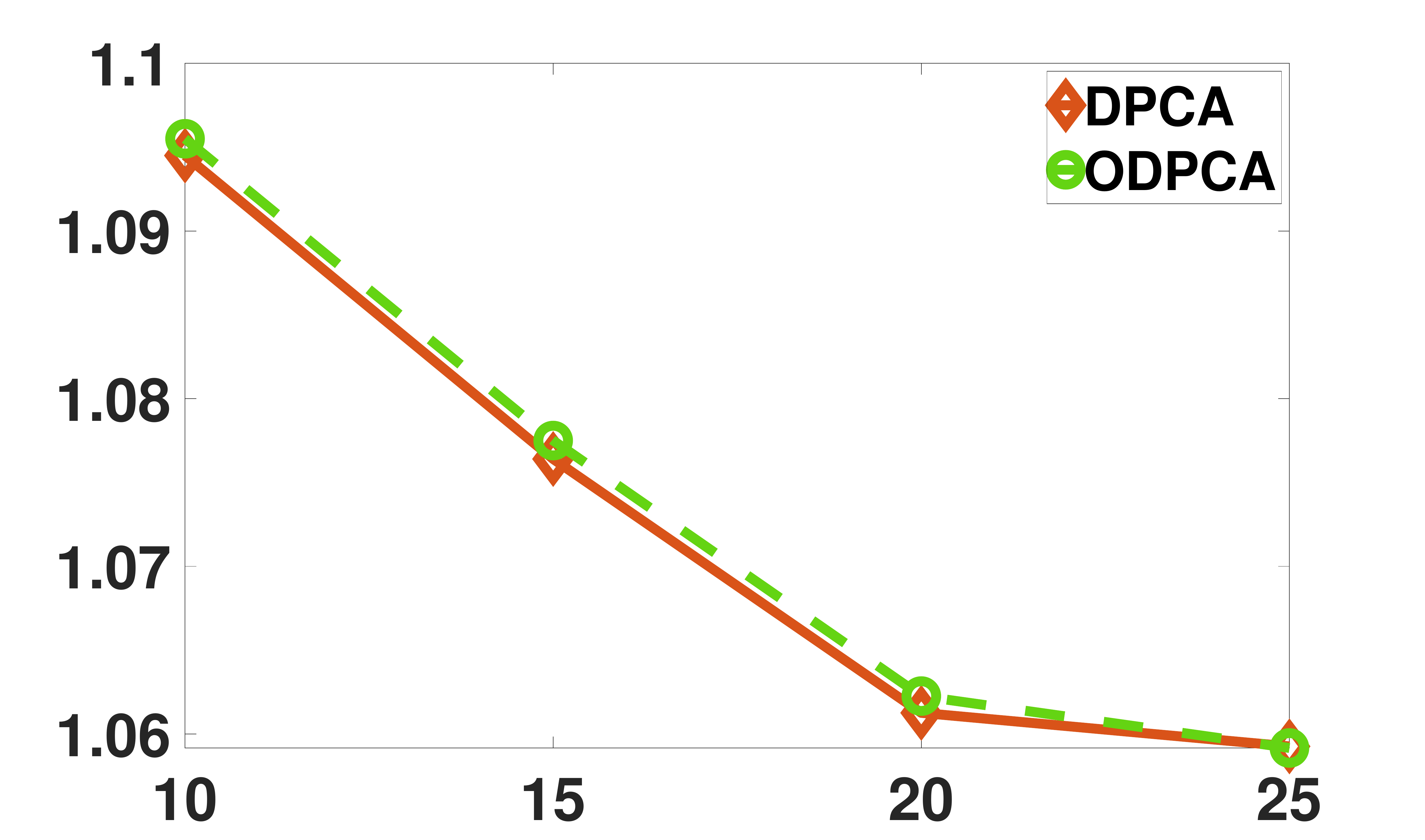}
\\
(b) ~ BOW NYTimes 
\\
\end{tabular}
}
\caption{$\boldsymbol{k}$-means Clustering: running time (left), and relative clustering cost (right) v.s. projection
dimension. }\label{fig:kmeans}
\end{figure}
Algorithm~\ref{algo:2} is implemented on a network of a star topology, comprising of $(m,T)=(5, 10)$ for NewsGroups and MNIST datasets, and $(m,T)=(5, 20)$ for BOW NYTimes. 

We analyze both the computational cost, and the learning accuracy, under different scenarios. The following distributed procedure on the global data is being used as the baseline for assessing the performance of both \textsc{Dpca} and \textsc{Odpca}. First, each local node computes all $d$ eigenvectors of its own data, and transmits them to the central node. Then, this node aggregates the information collected from the local nodes, and computes the top $K$ eigenvectors. 

The \textsc{Dpca} algorithm is implemented as follows. For each $\ell \in [m]$, the $\ell$-th node computes the projection matrix $\hat{\bV}_{K+Z}^{(\ell)}$ corresponding to the top $K+Z$ empirical eigenvectors of its own data, for some integer number $Z$. Then, the central node calculates $\bar{\bSigma}= m^{-1}\sum\limits_{i=1}^m \hat{\bV}_{K+Z}^{(\ell)} \hat{\bV}_{K+Z}^{(\ell)^\top}$, followed by another round of eigenvalue decomposition to obtain the $K$ leading eigenvectors of $\bar{\bSigma}$. A similar procedure is employed in the first round of \textsc{Odpca}, at every time $t \in [T]$. Intuitively, the goal of transmitting $Z$ additional or fewer eigenvectors is to demonstrate the computational costs of the algorithms. Henceforth, we refer to $K+Z$ as \textit{projection dimension}.   

For the low-rank approximation tasks, we report the estimation error of the solution obtained by \textsc{Dpca} and \textsc{Odpca}, normalized to that of the baseline algorithm described above. Fig.~\ref{fig:lowrank} shows the average results over $5$ replications. Note that the horizontal axis reflects the projection dimension $K+Z$. The left column illustrating the computational time suggests that \textsc{Odpca} is remarkably faster than \textsc{Dpca}. Indeed, the running time of \textsc{Odpca} improves over that of \textsc{Dpca} by a factor between 2 to 5. The column on the right-hand-side of Fig.~\ref{fig:lowrank} indicates the estimation performance. It is clear that the normalized error of \textsc{Odpca} is comparable to that of \textsc{Dpca}. Hence, the speed-up technique does not lead to any distinguishable sacrifice in the performance of \textsc{Odpca}.

Next, we illustrate the performance of \textsc{Odpca} on the distributed $\boldsymbol{k}$-means clustering. First, \textsc{Dpca} and \textsc{Odpca} compute the top $K$ eigenvectors, and send them to all nodes. Each node of the distributed $\boldsymbol{k}$-means algorithm projects its own data onto the new feature space, using the aforementioned $K$ eigenvectors. We run the algorithm in the work of Balcan et al.~\cite{balcan2013distributed} over $5$ replications. 

Similar to the previous task, we report the clustering cost of the solution obtained by \textsc{Dpca} and \textsc{Odpca}, to that of results provided by running Lloyd's method on the full data. Fig.~\ref{fig:kmeans} shows the results for $\boldsymbol{k}$-means clustering tasks. It can be easily seen that \textsc{Odpca}'s solution is approximately as accurate as \textsc{Dpca}'s one. Further, a large decrease is observed in the running time of \textsc{Odpca} compared to \textsc{Dpca}. 

Both figures indicate that the proposed local and global aggregation steps lead to fast computation of the principal eigenspaces, while preserving the accuracy almost the same. 

\bibliographystyle{IEEEtran}
\bibliography{ref}

\end{document}